\documentclass[sigconf, authorversion]{acmart}

\AtBeginDocument{%
  \providecommand\BibTeX{{%
    \normalfont B\kern-0.5em{\scshape i\kern-0.25em b}\kern-0.8em\TeX}}}

\copyrightyear{2021} 
\acmYear{2021} 
\setcopyright{acmlicensed}\acmConference[SIGMOD '21]{Proceedings of the 2021 International Conference on Management of Data}{June 20--25, 2021}{Virtual Event, China}
\acmBooktitle{Proceedings of the 2021 International Conference on Management of Data (SIGMOD '21), June 20--25, 2021, Virtual Event, China}
\acmPrice{15.00}
\acmDOI{10.1145/3448016.3457301}
\acmISBN{978-1-4503-8343-1/21/06}

\usepackage{algorithm}
\usepackage{algpseudocode}
\usepackage[caption=false]{subfig}
\usepackage{subfloat}
\usepackage{tabularx} %

\makeatletter
\newcommand{\multiline}[1]{%
	\begin{tabularx}{\dimexpr\linewidth-\ALG@thistlm}[t]{@{}X@{}}
		#1
	\end{tabularx}
}
\makeatother

\newtheorem{definition}{Definition}
\newtheorem{proposition}{Proposition}

\DeclareMathOperator{\wracc}{WRAcc}
\newcommand*\diff{\mathop{}\!\mathrm{d}}

\newcommand{\Np}{N^+}

\begin{document}
\fancyhead{}
\title{REDS: Rule Extraction for Discovering Scenarios}

\author{Vadim Arzamasov}
\affiliation{%
	\institution{Karlsruhe Institute of Technology}
	\streetaddress{P.O. Box 6980}
	\city{Karlsruhe}
	\country{Germany}
	\postcode{76049}
}
\email{vadim.arzamasov@kit.edu}

\author{Klemens B\"ohm}
\affiliation{%
	\institution{Karlsruhe Institute of Technology}
	\streetaddress{P.O. Box 6980}
	\city{Karlsruhe}
	\country{Germany}
	\postcode{76049}
}
\email{klemens.boehm@kit.edu}

\begin{abstract}
  Scenario discovery is the process of finding areas of interest, known as scenarios, in data spaces resulting from simulations. 
  For instance, one might search for conditions, i.e., inputs of the simulation model, where the system is unstable.
  Subgroup discovery methods are commonly used for scenario discovery.
  They find scenarios in the form of hyperboxes, which are easy to comprehend. 
  Given a computational budget, results tend to get worse as the number of inputs of the simulation model and the cost of simulations increase. 
  We propose a new procedure for scenario discovery from few simulations, dubbed REDS. 
  A key ingredient is using an intermediate machine learning model to label data for subsequent use by conventional subgroup discovery methods.
  We provide statistical arguments why this is an improvement.  
  In our experiments, REDS reduces the number of simulations required by 50--75\% on average, depending on the quality measure. 
  It is also useful as a semi-supervised subgroup discovery method and for discovering better scenarios from third-party data, when a simulation model is not available.
\end{abstract}



\maketitle

\section{Introduction}
\label{section:intro}

The behavior of many systems, such as electrical grids or climate systems, can be described with differential or difference equations. 
The resulting model connects a set of input values to the output and can be solved with computer experiments, aka.\ simulations.
The inputs of a simulation model can be classified into (1) \emph{control} variables, i.e., those that the user of the model (scientist, engineer, policy maker) can set, and (2) \emph{environmental} variables. 
They reflect uncertainty regarding specific conditions under which the modeled phenomena take place~\cite{Santner2003}. 
For instance, for an electricity producer, the output power and the energy consumption are control and environmental variables respectively.

Analyzing data resulting from simulations has been of interest to the data management community for a long time.
The metaphor ``data farming'' captures the tasks associated with such data~\cite{DBLP:conf/wsc/Sanchez18}.
After performing simulations for different input values, one often replaces the simulation model with a statistical or machine learning (ML) model, a so-called \emph{metamodel}~\cite{Gorissen2010,Uusitalo2015}.
For instance, when the simulation model does not contain environmental variables, or their distribution is known, a metamodel can help to optimize the simulated system \cite{Kleijnen2015,Simpson2001,Wang2007,Gorissen2010}.

However, if the distributions of environmental variables are not known, so-called \emph{deep uncertainty} occurs~\cite{Herman2015,Bryant2010,Walker2013}. 
In this case, the purpose of learning a metamodel is to understand the behavior of a simulated system. 
This process has recently become known as \emph{scenario discovery}. 
A \emph{scenario} is a region of particular interest in the space of environmental inputs~\cite{Kwakkel2016,Kwakkel2016a,Islam2016}. 
Following the usual definition, a region is interesting if the output variable is above some threshold or takes a particular value.

\begin{figure}[t]
	\centering
	\includegraphics[width=2.4in]{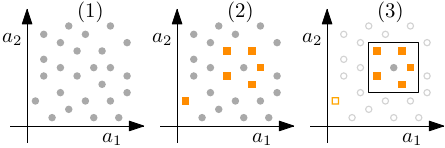}
	\caption{Scenario discovery process.}
	\label{sdprocess}
\end{figure}

One typically performs scenario discovery by
\begin{enumerate}
	\item running simulations for different values of environmental inputs $a_j$ drawn from a uniform distribution;
	\item labeling the outcomes of interest with $y=1$; the rest with 0; 
	\item applying a scenario discovery algorithm to find scenarios. 
\end{enumerate} 
The algorithm used in the last step often is PRIM~\cite{Friedman1999}, described in Section~\ref{section:prim}. 
It finds regions in the form of hyperboxes; see Figure~\ref{sdprocess}.
A hyperbox can be described in the form of a rule, for instance
$$\textrm{IF}\quad a_1^l\le a_1\le a_1^r \quad\textrm{AND}\quad a_2^l\le a_2\le a_2^r \quad\textrm{THEN}\quad y=1,$$
where $y$ is the simulation output and $a_1^l,a_1^r,a_2^l,a_2^r$ are real numbers.

The hyperbox should cover a big share of the interesting region, minimize the coverage of uninteresting space, and not restrict the inputs with little effect on the output. 
The respective metrics are recall, precision, and interpretability~\cite{Bryant2010}.
Since simulations often are computationally expensive~\cite{Wang2007,yilmaz2019reducing}, one wants to obtain a good scenario with few simulation runs. 

\begin{figure}[t]
	\centering
	\includegraphics[width=1.8in]{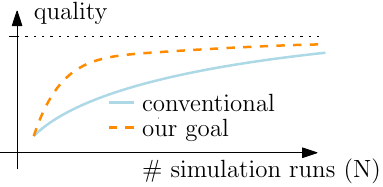}
	\caption{Scenario quality vs. number of simulations.}
	\label{fig:generalization_speed}
\end{figure}

In this paper, we address the problem of reducing the number of simulation runs to obtain high-quality scenarios or, equivalently, increasing the quality of scenarios discovered from a limited number of runs.
Our innovation is combining the conventional scenario-discovery approach with a powerful ML model. 
We will show that, with a conventional scenario-discovery procedure, quality increases slowly with the number of runs (Figure~\ref{fig:generalization_speed}). 
On the other hand, some ML models like random forest can learn a good approximation of a simulated model from a small dataset already.
We use these ML models to label more data for the scenario-discovery method.

Put differently, we propose a new scenario discovery process, REDS (Rule Extraction for Discovering Scenarios), with an intermediate step of estimating an accurate metamodel and using it to obtain a more extensive data set to find scenarios. Since any metamodel is an imperfect approximation of a simulation model, the labels it provides are noisy. 
We provide an analysis leading to the expectation that REDS works better than conventional scenario-discovery approaches from a statistical point of view.
We also discuss the suitability of subgroup-discovery algorithms for scenario discovery and explain that PRIM has certain advantages.
We then list requirements for meaningful evaluations and show that various existing studies do not fulfill them. 
In our experiments, which satisfy these requirements, we compare REDS to existing approaches, 
using various data sources from the metamodeling and scenario-discovery literature. 
REDS improves scenario discovery significantly.
Next, we show that REDS can find subgroups in a semi-supervised setting, and it improves scenario discovery for third-party data where a simulation model is not available.
Our code is openly available\footnote{\url{https://github.com/Arzik1987/REDS_experiments}}.

Paper outline: Section~\ref{section:relatedwork} reviews related work. Section~\ref{section:background} describes
existing algorithms for scenario discovery.
Section~\ref{section:QM} features quality measures.
Section~\ref{section:discussion} examines the suitability of existing subgroup-discovery methods for scenario discovery. 
Section~\ref{section:proposed} introduces our approach and justifies it
from a statistical point of view. 
Section~\ref{section:complexity} features complexity results.
Section~\ref{section:experimentalsetup} covers evaluation principles and our experimental setup. 
Section~\ref{section:results} features results. 
Section~\ref{section:future} describes future research. 
Section~\ref{section:conclusion} concludes.

\section{Related Work}
\label{section:relatedwork}
In this section, we describe scenario discovery techniques used in the literature and their recent improvements. 
Next, we discuss subgroup discovery methods as scenario discovery tools. 
Finally, we review work on weak supervision which has inspired our proposal.

\subsection{Scenario Discovery}
\label{subsection:scenario_discovery}
Scenario discovery targets at interpretable representations of simulation outputs. 
Whether a particular representation is easy to interpret depends on the task at hand~\cite{Freitas2013} and on the background of the user~\cite{Huysmans2011}. 
Empirical studies suggest that ML models outputting hyperboxes are interpretable~\cite{DBLP:journals/jmlr/BaehrensSHKHM10,Freitas2013,Huysmans2011}.
Four big groups of such models are decision trees~\cite{DBLP:books/wa/BreimanFOS84,DBLP:books/mk/Quinlan93}, classification rules~\cite{DBLP:conf/icml/Cohen95,DBLP:conf/aistats/MitaPFM20},  subgroups~\cite{atzmueller2015subgroup,DBLP:journals/kais/HerreraCGJ11} and association rules~\cite{DBLP:conf/sigmod/AgrawalIS93,DBLP:conf/vldb/AgrawalS94,DBLP:conf/sigmod/SrikantA96,DBLP:conf/sigmod/BrinMUT97,DBLP:conf/sigmod/HanPY00}; 
see also~\cite{DBLP:reference/ml/2010,furnkranz2012foundations} and references therein. Association rule learning is an unsupervised task. 
The methods from the other three groups have been used for scenario discovery. 
The articles~\cite{Pierreval1992,Yoshida1989,Berthold1999} use classification rules, \cite{lempert2008comparing, Hadka2015, arzamasov2018towards,Kwakkel2013} use decision trees. 
Currently, the subgroup discovery method PRIM dominates in the scenario-discovery domain, see, e.g.,~\cite{Dalal2013,Kwakkel2016,Kwakkel2016a,Bryant2010,Herman2015,Lempert2006,Groves2007,lempert2008comparing,Hadka2015,Kwakkel2013,ElineGuivarch2016}.
Several improvements of PRIM for scenario discovery have been proposed. 
They include using different target functions guiding the search process~\cite{Kwakkel2016a}, bumping~\cite{hastie2009elements} with random feature selection~\cite{Kwakkel2016}, and combining PRIM with principal component analysis (PCA-PRIM)~\cite{Dalal2013}.
We will compare our method REDS to PRIM with bumping~\cite{Kwakkel2016} and the original PRIM.
PCA-PRIM~\cite{Dalal2013} and different target functions~\cite{Kwakkel2016a} are compatible with REDS and orthogonal to our study.

\subsection{Subgroup Discovery}
\label{subsection:subgroup_discovery}
Early subgroup-discovery algorithms have been derived from existing classification or association rule learning methods or from decision trees~\cite{DBLP:journals/kais/HerreraCGJ11}. 
The task of subgroup discovery is to find hyperboxes separating the subgroups of examples that are large and have a distribution of the output $y$ that significantly differs from the one in the entire dataset. 
This is different from classification rule learning which targets at finding complete and consistent models, i.e., those covering most of the examples with $y=1$ and a very small number of examples with $y=0$~\cite{DBLP:journals/eswa/ValmarskaLFR17}.
In contract to classification rule learning, subgroup-discovery methods often (1) allow to control the number of subgroups~\cite{DBLP:journals/datamine/LeeuwenK12,DBLP:journals/datamine/GrosskreutzR09}, (2) focus on the properties
of individual rules rather than on properties of the rule set as a whole~\cite{DBLP:journals/eswa/ValmarskaLFR17,DBLP:journals/jmlr/LavracKFT04}, and (3) tolerate more false positives~\cite{DBLP:journals/jmlr/LavracKFT04} giving way to higher interpretability of scenarios. 
These qualities are desirable for scenario discovery~\cite{Bryant2010}. 
\cite{DBLP:journals/eswa/ValmarskaLFR17,DBLP:journals/kais/HerreraCGJ11,furnkranz2012foundations} explain relations between classification or association rule learning and subgroups discovery. \cite{lempert2008comparing,hastie2009elements} contrast decision trees with subgroups.

Numerous subgroup-discovery methods exist \cite{atzmueller2015subgroup,DBLP:journals/kais/HerreraCGJ11,DBLP:journals/jcst/Helal16}, 
but many of them do not support real-valued attributes. 
Some methods that accept continuous inputs take hours or days to find subgroups even in small data sets~\cite{DBLP:conf/pakdd/MillotCB20,DBLP:journals/datamine/GrosskreutzR09}.
From the remaining ones, PRIM stands out as it allows a user to balance between interpretability (the number of inputs in the rule description), precision, and recall of each box, and provides visualizations.
This trade-off option is important for scenario discovery~\cite{Bryant2010}, as we show in Section~\ref{section:discussion}. 
It might also be a reason why PRIM is one of the most cited subgroup-discovery algorithms. 
To demonstrate the generality of our proposal to some extent, we show that our method also improves the output of another subgroup-discovery algorithm, \textsc{BestInterval} (BI)~\cite{MampaeyNFK12}.

\subsection{Weakly Supervised Learning}
\label{subsection:weak_learning}
Weakly supervised learning addresses different supervision deficiencies, e.g., incomplete or inaccurate supervision~\cite{DBLP:journals/corr/abs-2012-09632}. 
One can perceive learning from few simulations as incomplete supervision since the labels of many feasible input combinations remain unknown. 
On the other hand, our method REDS is an instance of inaccurate supervision because the pseudo-labels from intermediate metamodel are noisy. 
We now review weak supervision approaches.

\subsubsection{Rule Extraction and Knowledge Distillation}
In a nutshell, to achieve interpretability, 
one can either learn an interpretable model or learn a complex (black-box) ML model and explain it. 
Guidotti et al.~\cite{DBLP:journals/csur/GuidottiMRTGP19} review methods for explaining black-box ML models. 
They distinguish between local and global explanations. The former ones~\cite{DBLP:conf/aaai/Ribeiro0G18,DBLP:journals/corr/abs-1805-10820} answer questions like ``Why does the metamodel predict this label for that input?'', global explanations questions like ``Where is the system stable?''.
Methods for global explanation often are rule-extraction algorithms~\cite{DBLP:journals/csur/GuidottiMRTGP19,Huysmans2006}.
The examples are \textsc{Trepan}~\cite{Craven1996} and CMM~\cite{Domingos1997}, which learn an $m$-of-$n$ decision tree and C4.5 rules from the output of an artificial neural network and an ensemble of C4.5 rules respectively.
Rule extraction can be seen as a subdomain of model parroting~\cite{Settles2010,Bucilua2006} or knowledge distillation~\cite{HintonVD15}, where one trains a better (e.g., faster and/or smaller) student model using another model as a teacher.
This often leads to better results than learning a student model directly from the initial data set~\cite{DBLP:journals/tnn/FortunyM15}.
In the research mentioned, the student model approximates any function with arbitrary accuracy given enough data. 
With REDS, the student model is a single or a small set of hyperboxes. 
It is not a ``universal approximator'', and its quality should be assessed differently from the one of a teacher, see~\cite{DBLP:journals/kais/HerreraCGJ11} for common subgroup-quality measures.

\subsubsection{Data Augmentation}
One can increase the size of a small dataset by augmenting it with new examples through applying label-preserving transformations to the existing ones. 
Examples of such transformations are crop or rotation for image data~\cite{DBLP:journals/jbd/ShortenK19}; replacing words with synonyms for text data~\cite{DBLP:conf/emnlp/WangY15}, or masking a
specific frequency channel for speech data~\cite{DBLP:conf/interspeech/ParkCZCZCL19}. 
For simulated data, label-preserving transformations are not obvious.
More sophisticated data-augmentation approaches use generative adversarial networks (GANs)~\cite{DBLP:journals/jbd/ShortenK19}. 
However, GANs have several problems when learning from small tabular data, e.g., do not converge or run into mode collapse~\cite{DBLP:conf/nips/SrivastavaVRGS17}.
GANs also do not allow to control the distribution of data points that has to be uniform in scenario discovery.

\subsubsection{Self-training and Multi-view Learning}
With self-training, an ML model is repeatedly trained on original data and its own predictions from the previous iteration. 
In multi-view learning, several models trained on small labeled data improve each other by labeling more examples. 
Both methods are instances of semi-supervised learning~\cite{zhu2005semi,DBLP:conf/cvpr/XieLHL20}. 
Unlike self-training, REDS is not iterative, and a student model is different from a teacher. 
In contrast to multi-view methods, the learning in REDS is directed: A subgroup discovery algorithm is always a student.

\section{Subgroup Discovery}
\label{section:background}
This section covers existing subgroup-discovery algorithms.

\subsection{Notation}
\label{subsection:notations}
Let $D$ be a dataset obtained from $\lvert D\rvert=N$ simulations:
$$D=\begin{pmatrix}
x_{11} &  \dots & x_{1M} & y_1 \\
\vdots &  \ddots & \vdots & \vdots \\
x_{N1} & \dots & x_{NM} & y_N \\
\end{pmatrix}.$$
The first $M$ columns contain the values $\left(x_{1j}\dots,x_{Nj}\right)$ of the inputs $a_j$ of a simulation model, and the last column $\left(y_1,\dots,y_N\right)$ is the observed simulation output. 
In the scenario-discovery domain, $y_i\in\left\{0,1\right\}$, i.e., one is interested in binary questions like ``when is the output under a certain value'' or ``for which input values does one policy outperform another one''. 
We define $\Np=\sum y_i$; hereafter the summation is from 1 to $N$ unless otherwise stated.
We refer to the first $M$ elements in each row $x_i=\left(x_{i1},\dots,x_{iM}\right)$ as a \emph{point}; the entire row $d_i=\left(x_i,y_i\right)$ is an \emph{example}. 
Stochastic simulation models implicitly define an unobservable function $f(x):\mathbb{R}^M\rightarrow[0,1]$, $f(x_i)=P(y_i=1|x_i)$; for deterministic models $f(x_i)=y_i$. 
A \emph{hyperbox} $B$ is a conjunction of intervals $B=\prod_{j=1}^{M}[a_j^l,a_j^r]$, $a_j^l\in\mathbb{R}\cup{-\infty}$, $a_j^r\in\mathbb{R}\cup{+\infty}$. 
Given a dataset $D$, $B$ defines a subgroup, a set of examples from $D$ 
within $B$:
$d_i\in B\Leftrightarrow d_i\in D\land$ $x_{ij}\in[a_j^l,a_j^r], \forall j$. 
The size of a subgroup and the number of ``interesting'' examples in it are $n=\sum\mathbb{I}(d_i\in B)$ and $n^+=\sum y_i\cdot\mathbb{I}(d_i\in B)$. Here $\mathbb{I}(z)$ is 1 if $z$ is true, and 0 otherwise. 
A quality measure of a subgroup is a function $\phi(B,D)$. 
In many cases, $\phi(B,D)=\phi(N,\Np,n,n^+,a_j^l,a_j^r)$, $j=1,\dots,M$.
We say that the input $a_j$ \emph{defines} a subgroup $B$ or that $a_j$ is \emph{restricted} if $a_j^l\ne-\infty$ or $a_j^r\ne+\infty$.

\subsection{Subgroup Discovery Algorithms}
\label{section:prim}
This section describes the algorithms PRIM, PRIM with bumping, and BI. These algorithms accept continuous and discrete inputs. 

\subsubsection{PRIM} 

\algsetblock{Procedure}{EndProcedure}{}{0.4cm}
\begin{algorithm}[t]
	\caption{PRIM, peeling step}\label{algorithm:primpeel}
	\begin{algorithmic}[1]
		\Procedure{PRIM.peel}{$D$, $D^{val}$, $\alpha$, $mp$} 	
		\State \multiline{Start with a train dataset $D$, a validation dataset $D^{val}$, and a list $Blist$ containing a single box $\prod_{j=1}^{M}[-\infty,+\infty]$;} 
		\State \label{alg:primpeel:step2} \multiline{%
			for each dimension $a_j$, $j=1,\dots,M$, create two candidate boxes from the last box in $Blist$ by cutting off a share $\alpha$ of points of $D$ inside it with the highest or the lowest values of $a_j$.
			Choose the candidate box $B$ with the highest value of $n^+/n$ on $D$ and append $B$ to $Blist$;}
		\State \label{alg:primpeel:step3} \multiline{repeat Step~\ref{alg:primpeel:step2} as long as the number of points of $D$ or $D^{val}$  in the box is at least ${mp}$;}
		\State \multiline{return the hyperbox with the highest value of $n^+/n$ on $D^{val}$ and all preceding boxes together with quality metrics.}
		\EndProcedure	
	\end{algorithmic}
\end{algorithm}

\begin{figure}[t]
	\centering
	\includegraphics[width=3.1in]{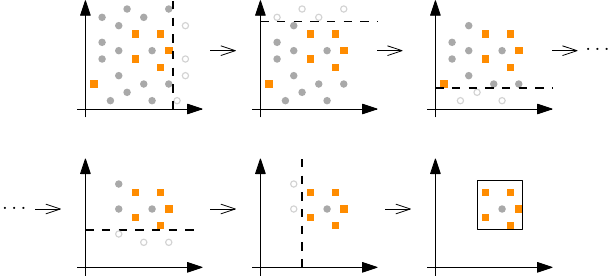}
	\caption{PRIM at work}
	\label{fig:prim_at_work}
\end{figure}

The Patient Rule Induction Method (PRIM) was originally proposed in \cite{Friedman1999}; \cite{hastie2009elements}
is a concise description. 
PRIM works in two phases, called peeling and pasting.
The latter had a negligible effect in our experiments, and many studies \cite{Kwakkel2016,Kwakkel2016a,Dalal2013} do not use it. 
Hence we describe only the peeling phase, cf. Algorithm~\ref{algorithm:primpeel}.
Figure~\ref{fig:prim_at_work} illustrates Step~\ref{alg:primpeel:step3} of the process.
Pasting works similarly but in the opposite direction~\cite{Friedman1999}. 
Each run of PRIM results in a sequence of nested boxes. From this sequence, domain experts choose a single box which best suits their needs. To find more subgroups, one applies the covering approach~\cite{Friedman1999,DBLP:journals/datamine/GrosskreutzR09}, i.e., repeatedly runs PRIM on subsets of the data which do not belong to previously discovered subgroups. 
One can use this covering approach with all algorithms described in this section to obtain a desired number of subgroups.

\subsubsection{PRIM with Bumping}

\begin{algorithm}[t]
	\caption{PRIM with bumping}\label{algorithm:primbag}
	\begin{algorithmic}[1]
		\Procedure{PRIM.bumping}{$D$, $m$, $Q$, $\dots$} 	
		\State $\textit{boxes}=\{\}$
		\For{$1<i<Q$}
		\State $D^\textit{bs}$ \Comment{a bootstrap sample of examples from $D$}
		\State $A=\{a_{j_1}, \dots,a_{j_m}\}$ \Comment{a random subset of $m$ inputs}
		\State $D^\textit{bs}_A$ \Comment{a submatrix of $D^\textit{bs}$ with columns in $A$}
		\State $\textit{boxes}_i=\textsc{PRIM.peel}(D = D^\textit{bs}_A,\dots)$
		\State $\textit{boxes}=\textsc{Append}(\textit{boxes},\textit{boxes}_i)$
		\EndFor
		\State \Return{$\textsc{NotDominated}(\textit{boxes})$}
		\EndProcedure	
	\end{algorithmic}
\end{algorithm}

The PRIM algorithm with bumping~\cite{Kwakkel2016} produces multiple boxes by varying a data set $D$ and returns only the ones not dominated by any other box in terms of precision and recall. We define these measures in Section~\ref{section:QM}.
\begin{definition}
	For a set of quality measures $\phi_1, \dots, \phi_q$ and a dataset $D$, a box $b$ is dominated by a box $B$ if $\forall k=1,\dots, q:$ $\phi_k(b,D)\le\phi_k(B,D)$ and $\exists k^*:\phi_{k^*}(b,D)<\phi_{k^*}(B,D)$.
\end{definition}
The PRIM algorithm with bumping (Algorithm~\ref{algorithm:primbag}) takes a random bootstrap sample $D^{bs}$ from $D$, a random subset of inputs $A\subseteq\{a_1,\dots,a_M\}$, $\lvert A\rvert = m$ (Lines~4--5). Then it 
runs the algorithm \textsc{PRIM.peel} with $D=D^{bs}$ using only columns in $A$ (Lines~6--7).
It repeats the above procedure $Q$ times (Line~3) and stores all boxes returned by \textsc{PRIM.peel} in the set $\textit{boxes}$ (Line~8). 
In the end, it excludes from $\textit{boxes}$ the hyperboxes dominated by any other box from the same set on the validation data $D^{val}$ and returns the result (Line~10).

\subsubsection{BI Algorithm}

\begin{algorithm}[t]
	\caption{The BI algorithm}
	\label{algorithm:BSBI}
	\begin{algorithmic}[1]
		\Procedure{BI}{$D, m, {bs}$}
		\State $k=1$, $Bset_0=\emptyset$
		\State ${Bset_1}=\{\prod_{j=1}^{M}[-\infty,+\infty]\}$
		\While{$Bset_k\ne Bset_{k-1}$}
		\For{$B\in{Bset_k}$}
		\For{$1\le j\le M$}
		\State $B_j=\textsc{BestIntervalWRAcc}(D,B,a_j)$
		\If{$\textsc{Restricted}(B_j)\le m$}
		\State ${Bset_k}={Bset_{k-1}}\cup B_j$
		\EndIf
		\EndFor
		\EndFor
		\State $k=k+1$
		\State ${Bset_{k}}=\textsc{Top}({Bset_k},{bs})$
		\EndWhile
		\State \Return{$\textsc{Top}({Bset_k},1)$}
		\EndProcedure	
	\end{algorithmic}
\end{algorithm}

\begin{figure}[t]
	\centering
	\includegraphics[width=2.4in]{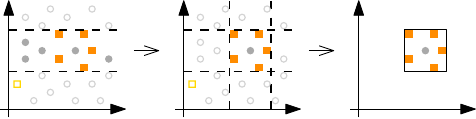}
	\caption{BI at work}
	\label{fig:bi_at_work}
\end{figure}

Algorithm~\ref{algorithm:BSBI} is the BI algorithm. 
It maximizes the WRAcc measure, cf.\  Section~\ref{section:QM}. 
BI takes a data set $D$, a ``depth'' parameter $m$ specifying the maximal number of restricted inputs (Line~8), and $bs$, the ``beam size'', i.e., the maximal number of candidate hyperboxes maintained during the search (Line~14). The procedure $\textsc{Restricted}(B_j)$ returns the number of inputs defining $B_j$; 
the procedure $\textsc{Top}({Bset_k},{bs})$ keeps $bs$ boxes with the highest WRAcc.
BI starts with a set containing a single box $\prod_{j=1}^{M}[-\infty,+\infty]$ (Line~3). 
It then iteratively tries to refine each box from the current set along one dimension (Lines~5--12) by invoking \textsc{BestIntervalWRAcc} (Line~7).
This subroutine updates the boundaries $a_j^l$, $a_j^r$ of $B$ for a specified input $a_j$ to maximize WRAcc, see~\cite{MampaeyNFK12} for more details.
BI stops when no further refinement is possible and returns a box with the highest WRAcc on $D$ (Line~16). 
Figure~\ref{fig:bi_at_work} shows how BI finds the subgroup by consecutively running \textsc{BestIntervalWRAcc} for two dimensions.  
As with PRIM, with BI one can find a required number of subgroups following the covering approach.

\section{Quality Metrics}
\label{section:QM}
In this section, we describe the quality metrics we use. Some of them, precision, WRAcc, and interpretability, are commonly used in the subgroup discovery domain. We propose three additional metrics: number of irrelevantly 
restricted inputs, consistency, and PR AUC.
The first one complements the interpretability measure.
Consistency quantifies the robustness of a scenario with respect to changes in data used to discover it. 
PR AUC allows assessing several nested boxes returned by PRIM with a single value.

The metrics evaluate the output from a single run of a subgroup discovery algorithm, i.e., one box for BI or a sequence of nested hyperboxes for PRIM.
To assess the quality of a set of subgroups after the covering approach (see, Section~\ref{section:prim}), one usually averages the qualities of the individual boxes the set consists of~\cite{DBLP:journals/datamine/GrosskreutzR09,DBLP:journals/jmlr/LavracKFT04}. 

\textbf{Precision, PR AUC}.
\label{subsection:CDI}
A scenario should contain many interesting examples ($y=1$) and few uninteresting ones ($y=0$). 
This is equivalent to high recall and precision~\cite{Bryant2010}:
$$\textit{recall}=\frac{n^+}{\Np}\qquad \textit{precision} = \frac{n^+}{n}.$$

\begin{figure}\centering
	\includegraphics[width=2.8in]{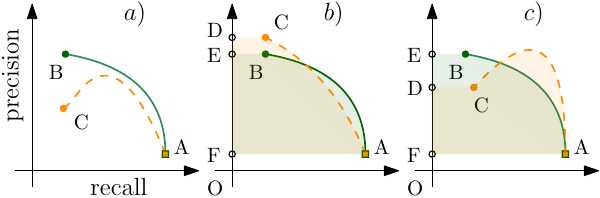}
	\caption{Mutual positions of two peeling trajectories}
	\label{fig:curves}
\end{figure}

PRIM outputs a sequence of boxes from which a user can choose one, usually by compromising between precision and recall~\cite{Bryant2010}. 
To exclude this subjective choice from our evaluation, we compute precision and recall for each box. 
The resulting pairs of values form a curve in the precision-recall coordinates,
the \emph{peeling trajectory}~\cite{Bryant2010,Friedman1999}.
To rank two algorithms, we compare their curves: $\textit{AB}$ and $\textit{AC}$ on Figure~\ref{fig:curves} with the area under the curve metric: 
We compare the areas covered by figures $\textit{ABEF}$ and $\textit{ACDF}$, as shown in Plots~b) and c). 
A larger area corresponds to a better algorithm. While this PR AUC measure is popular in machine learning~\cite{DBLP:conf/icml/DavisG06,saito2015precision}, to our knowledge, it has not been used for PRIM output so far.

Sometimes one wants to find scenarios as pure as possible, i.e., allow lower recall for high precision. 
Since test data is not available in reality, one can make this choice using validation data $D^{val}$. 
This corresponds to choosing the last box returned by PRIM.

\textbf{WRAcc}. The BI algorithm outputs a single box that optimizes the Weighted Relative Accuracy (WRAcc):
$$\wracc=\frac{n}{N}\left(\frac{n^+}{n}-\frac{\Np}{N}\right)$$ 
It tends to favor large boxes, with high $n/N$, which also have a high mean value of output $y$, i.e, high $n^+/n$.
So we will use WRAcc to assess the output of BI.

\textbf{Interpretability}. 
Research on interpretability and on ways to quantify it is ongoing, see~\cite{rudin2019stop,DBLP:journals/ai/Miller19,DBLP:journals/cacm/Lipton18}.
In scenario discovery, one defines interpretability as the number of inputs restricted by the hyperbox
defining the scenario~\cite{Bryant2010,Kwakkel2016}: 
$$\# \textit{restricted} = \sum_{j=1}^M\mathbb{I}\left(a_j^l\ne -\infty\lor a_j^r\ne +\infty\right).$$
Low values of $\# \textit{restricted}$ mean high interpretability.
For PRIM, we will compute this and the following measures using the last box. 

\textbf{Number of Irrelevantly Restricted Inputs}.
It is not only important to have few restricted inputs. One also does not want any input of a simulation model without any influence on the output to define a subgroup. 
For our experiments, we propose a respective measure that counts the number of irrelevantly restricted inputs:
$$\# \textit{irrel} = \sum_{j=1}^M\mathbb{I}\left(\left[a_j^l\ne -\infty\lor a_j^r\ne +\infty\right]\land P\left(y=1\lvert a_j\right)=P\left(y=1\right)\right).$$
We use it as a proxy of subgroup compliance with the ``prior knowledge'' of an expert, another dimension of  interpretability~\cite{pazzani1997comprehensible}.

\textbf{Consistency}.
Different data sets produced by a simulation model can result in different scenarios discovered.
This reduces interpretability \cite{Friedman1999} since one looks for some hidden structure in the model, rather than in the particular data produced by it. So we introduce another quality measure for scenarios, \emph{consistency}. 
\begin{definition}
	For two data sets $D_1$ and $D_2$, with $\lvert D_1\rvert=\lvert D_2\rvert$, produced by the same model with continuous inputs, let $B_1$ and $B_2$ be the scenario descriptions obtained with the same scenario-discovery algorithm $\textit{SD}$. 
	Let $V_o$ be the volume of the overlap of these boxes and $V_u$ the one of the union of $B_1$ and $B_2$. 
	The consistency of $\textit{SD}$ is 
	$$\textit{consistency}=\mathbb{E}\left[{V_o}/{V_u}\right]$$      
\end{definition}
When the definition of $B_1$ or $B_2$ includes unbounded intervals ($a_j^l=-\infty$ or $a_j^r=+\infty$), we replace infinities with the minimal and maximal values of the respective input. 
For discrete inputs, one can use counts of distinct values instead of volumes. 

Consistency is often used in the rule learning literature, but with other definitions.
For instance, \cite{Huysmans2006} defines it with respect to the inherent randomness of the algorithm.

\section{PRIM vs. BI} 
\label{section:discussion}
The ``interactivity'' of PRIM is desirable for scenario discovery~\cite{Bryant2010}. We now explain why it is an advantage with new arguments.

After a single run, BI returns a hyperbox, balancing precision and recall with the WRAcc measure. 
Most other subgroup-discovery methods do the same, sometimes using different measures.
PRIM in turn lets a user choose one out of a set of nested hyperboxes, displaying their precision and recall values.
Now we say why this interactive behavior is desirable. 
Consider the following example.

\begin{example}
	The output $y$ of a simulation model is a function of its single input $a$ and is defined on some interval $[0,h]$, $h\in\mathbb{R}$:
	
	\begin{align}
		f(a)= P(y=1|a) = \left\{ \begin{array}{lc} 
			1 & \hspace{5mm} a\in[0,1) \\
			a-1 & \hspace{5mm} a\in[1,2] \\
			0 & \hspace{5mm} a\in(2,h] \\
		\end{array} \right. \nonumber
	\end{align}
	Depending on user needs, two intervals might be of interest. 
	The interval $[0,1]$ has the highest precision 1, as $y=1$. 
	Another interval, $[0,2]$, contains all points with $f(a)>0$. Suppose that one runs many simulations, i.e., $N\rightarrow\infty$. 
	PRIM will output a sequence of intervals, containing the two just mentioned.
	BI will return a single box that maximizes WRAcc. 
	For the above function, 
	$$\wracc([0,1])=\frac{1}{h}-\frac{1.5}{h^2};\quad \wracc([0,2])=\frac{1.5}{h}-\frac{3}{h^2}.$$
	$\wracc([0,1])>\wracc([0,2])\Leftrightarrow h<3$, i.e., the result depends on interval boundary $h$. 
\end{example}

The input ranges, $h$ in the example, are often arbitrary to some extent in scenario discovery. 
One wants to avoid a dependency of a scenario on these ranges. 
BI and other methods do not facilitate this. 
In contrast, a user can easily find both intervals from the example regardless of the value of $h$.
They manifest themselves as sudden changes in the slope of a peeling trajectory.

\section{Proposed Method: REDS} 
\label{section:proposed}

\begin{algorithm}[t]
	\caption{Our method: REDS}\label{primrf}
	\begin{algorithmic}[1]
		\Procedure{REDS}{$D,AM,L,\textsc{SD},\dots$}
		\State $f^{am}(x) = \textsc{Train}(AM,D)$
		\State $D^\textit{new}=\textsc{NewPoints}(L)$
		\For{$1 < i < L$}
		\State $y^\textit{new}_i=\mathbb{I}\left(f^{am}(x^\textit{new}_i)>\textit{bnd}\right)$
		\EndFor
		\State \textbf{return} $\textsc{SD}(D^\textit{new},\dots)$
		\EndProcedure
	\end{algorithmic}
\end{algorithm}

This section features our new method, REDS. From a statistical point of view, we also show why conventional subgroup discovery methods have difficulties in learning functional scenarios from small data sets, and how REDS can overcome them.

\subsection{REDS Algorithm}
\label{subsection:reds}
The novelty of REDS is introducing an intermediate step to the conventional scenario-discovery approach described in Section~\ref{section:intro}. 
In this step, we train a metamodel with low variance and generalization error.
The proposed process (Algorithm~\ref{primrf}) is as follows. 
\begin{enumerate}
	\item Use $D$ to train an accurate metamodel $AM$ (Line~2); 	
	\item sample $L$ points i.i.d. from the same distribution $p(x)$ as points in $D$ to create $D^\textit{new}$ (Line~3);
	\item label $D^\textit{new}$ with the trained metamodel $f^{am}$ (Lines~4--6); 
	\item apply a scenario discovery algorithm $\textsc{SD}$ on $D^\textit{new}$ (Line~7). 
\end{enumerate} 
$\textsc{SD}$ can be any suitable algorithm, for instance, PRIM, PRIM with bumping, or BI. REDS accepts the same data types as $\textsc{SD}$.
The value of $\textit{bnd}$ in Line~5 depends on the metamodel $AM$; it is used to ensure that $y^\textit{new}\in\{0,1\}$.
In this paper, we experiment with random forest~\cite{DBLP:conf/icdar/Ho95,DBLP:journals/ml/Breiman01}, XGBoost~\cite{DBLP:conf/kdd/ChenG16}, and support vector machine (SVM)~\cite{DBLP:journals/ml/CortesV95} as $AM$ since they perform well in various tasks~\cite{wainberg2016random,DBLP:conf/gecco/OrzechowskiCM18}.
If $AM$ is random forest or XGBoost, then $f^{am}(x)$ (Line~2) approximates the unobserved function $f(x)=P(y=1|x)$, and one can modify Algorithm~\ref{primrf} by replacing Line~5 with $y^\textit{new}_i=f^{am}(x^\textit{new}_i)$. In some cases this modification achieves the best results, as we will show.

Under deep uncertainty, one assumes a uniform distribution of model inputs, $p(x)=\textrm{const}$. 
For our following derivations, we only use that $p(x)$ is known. 
This relaxation also allows one to see REDS as a semi-supervised subgroup discovery algorithm, i.e., able to learn from labeled and unlabeled data~\cite{zhu2005semi}. 
In particular, one can use the whole dataset instead of $D^\textit{new}$ (Line~3); we only require that points in labeled and unlabeled parts come from the same $p(x)$.

\subsection{Statistical Intuition behind REDS}
\label{subsection:statistical_intuition}
In the rest of this section, we assume $\textsc{SD}$ to be PRIM.
In each iteration $k$, PRIM shrinks the box $B_{k}$ to obtain the box $B_{k+1}$. 
It aims at choosing $B_{k+1}$ from candidate boxes $B_\textit{jk}$ with $j=1,\dots,2M$ so that the mean value of $f(x)$ in $B_{k+1}$ is maximal.
Equivalently, the mean value of $f(x)$ in the box $b_\textit{jk}=B_{k}\setminus B_\textit{jk}$ which is ``peeled off'' is minimal. 
In reality, $f(x)$ is unknown, and its mean $\mu_\textit{jk}$ in $b_\textit{jk}$ is estimated from the sample of points in $b_\textit{jk}$. 
A high error of this estimate may result in cutting off the wrong box, the one which does not maximize the mean value of $f(x)$ in $B_{k+1}$. 
So REDS will likely make fewer wrong cuts if its error in estimating $\mu_\textit{jk}$ is smaller. 

Let $b\in \{b_\textit{jk}\}$. The mean value of $f(x)$ in $b$ is
\begin{equation}
\label{eq:1}
\mu=\frac{\int_b f(x)p(x)\diff x}{\int_b p(x)\diff x}.
\end{equation} 
Here $p(x)$ is the pdf of the $M$-dimensional random variable $X$ denoting the point in the input space, as before. 
In the following analysis, we assume $b$ to be fixed. It contains $n'=\alpha \cdot (1-\alpha)^{k} \cdot N$ points
labeled with $y_i$, $i=1,\dots,n'$, by means of simulations. 
Here $N$ is the number of data points in $D$, as before.
The estimate of mean $\mu$ from the data is 
\begin{equation}
\hat{\mu} = \frac{1}{n'}\sum y_i\cdot\mathbb{I}(d_i\in b)
\label{eq:conventional}
\end{equation} 
The mean squared error (MSE) of this quantity is~\cite{Friedman1999} 
\begin{equation}
\begin{split}
\textit{MSE}_O=\mathbb{E}[(\mu-\hat{\mu})^2]=(\mu-\mathbb{E}[\hat{\mu}])^2+\mathbb{E}[(\hat{\mu}-\mathbb{E}[\hat{\mu}])^2]\\
=\textrm{Bias}^2(\hat{\mu})+\textrm{Var}(\hat{\mu})
\end{split}
\label{eq:bvdec}
\end{equation}
Here the expectation is taken over all datasets $D$ with $\lvert D\rvert=N$ containing points i.i.d. from $p(x)$. 
In this case, $\hat{\mu}$ is an unbiased estimate of $\mu$.
Remember that $y_i$ takes values from $\{0,1\}$. 
Thus, $y$ is a Bernoulli random variable with $P\left(y=1\vert x\in b\right)=\mu$. 
Formally, 
\begin{equation}
\textrm{Bias}(\hat{\mu})=0, \quad 
\textrm{Var}(\hat{\mu})=\frac{\mu(1-\mu)}{n'}=\textit{MSE}_O
\label{eq:bv_prim}
\end{equation}

With REDS, a function $f^{am}(x)$ learned with metamodel $AM$ is used to label points. 
Let
\begin{equation}
\mu^{am}=\frac{\int_b y^\textit{new}(x)p(x)\diff x}{\int_b p(x)\diff x}
,\quad
\hat{\mu}^{am}=\frac{1}{l}\sum_{i=1}^L y_i^\textit{new}\cdot\mathbb{I}(d_i^\textit{new}\in b)
\label{eq:mean_CM}
\end{equation}
where $l=\sum_{i=1}^L\mathbb{I}(d_i^\textit{new}\in b)\approx n'\cdot L/N$ is the number of newly generated points inside $b$, and $y^\textit{new}(x)=\mathbb{I}(f^{am}(x)>\textit{bnd})$. 
In general, $f(x)\ne y^\textit{new}(x)$ and $\mu^{am}\ne\mu$. 
Assume first a \emph{fixed} function $f^{am}(x)$. 
It generally also implies a fixed $D$.
Analogously to~(\ref{eq:bvdec})--(\ref{eq:bv_prim}), the bias-variance decomposition of the mean squared error is 
\begin{equation}
\begin{split}
\textrm{Bias}^2\left(\hat{\mu}^{am}\vert f^{am}(x)\right)=\left(\mu-\mu^{am}\right)^2, \\
\textrm{Var}\left(\hat{\mu}^{am}\vert f^{am}(x)\right)=\frac{\mu^{am}\left(1-\mu^{am}\right)}{l}\xrightarrow{l\to\infty}0
\label{eq:bv_prim_ml}
\end{split}
\end{equation}
where the expectation was taken over all datasets $D^\textit{new}$  (Algorithm~\ref{primrf}, Line~3) that are possible with our approach. 
Here we used that
points in $D^\textit{new}$ come from $p(x)$, hence $P\left(y^\textit{new}=1\vert x^\textit{new}\in b\right)=\mu^{am}$.
When $D$ and $f^{am}(x)$ vary, MSE with REDS for large $L$ is 
\begin{equation}
\textit{MSE}_\textit{R}=\mathbb{E}[(\mu - \mu^{am})^2],
\label{eq:mse_am}
\end{equation}
where the expectation is taken over all feasible datasets $D$ as in~(\ref{eq:bvdec}), and all possible fits $f^{am}(x)$ of a given metamodel $AM$ obtained using these datasets.

Now we can compare the $\textit{MSE}_O$ obtained with PRIM~(\ref{eq:bv_prim}) 
with $\textit{MSE}_\textit{R}$ with REDS~(\ref{eq:mse_am}). 
Lower MSE values mean a higher probability of finding a box $B_{k+1}$ with optimal precision \emph{and} recall values and lead to better values of precision and the PR AUC metrics.
Assuming that the best scenario is the one discovered with PRIM knowing the true function $f(x)$, REDS will perform superior to plain PRIM if, for all possible boxes~$b$, $\mathbb{E}[(\mu - \mu^{am})^2]<\mu(1-\mu)/n'$. Similarly, our method is \emph{likely} to show better performance than the original PRIM if the above inequality holds for the \emph{majority} of boxes.
The left-hand side of the inequality implicitly depends on $N$, since increasing the size of a training set typically leads to a better approximation of $f(x)$ by $f^{am}(x)$ and a lower $\mathbb{E}[(\mu - \mu^{am})^2]$.

Now consider the modification of REDS from Section~\ref{subsection:reds} where  $y^\textit{new}_i=f^{am}(x^\textit{new}_i)$, i.e., $y_i^\textit{new}\in[0,1]$ and $y^\textit{new}=f^{am}(x)$ in~(\ref{eq:mean_CM}). Interestingly, in this case, our approach may outperform the original one even when the size $L$ of the new dataset $D^\textit{new}$ is comparable to the size $N$ of the initial dataset~$D$. Specifically, the following holds.

\begin{proposition} 
	If $n'=l\land\mu=\mu^{am}$, then $\textrm{Var}(\hat{\mu}^{am})\le\textrm{Var}(\hat{\mu})$
\end{proposition}
\begin{proof} 
	Since $n'=l$, it is sufficient to show that within $b$
	\begin{equation}
	\mathbb{E}[(y^\textit{new})^2]-(\mathbb{E}[y^\textit{new}])^2=\textrm{Var}(y^\textit{new})\le \textrm{Var}(y)=\mu(1-\mu)
	\label{eq:proof1}
	\end{equation}
	Since $\mathbb{E}[y^\textit{new}]=\mu^{am}=\mu$, $(\ref{eq:proof1})\Leftrightarrow \mathbb{E}[(y^\textit{new})^2]\le\mu$.
	This is true as
	\begin{equation}
	\begin{split}
	\mathbb{E}[(y^\textit{new})^2]=\int_0^1 (y^\textit{new})^2g(y^\textit{new})\diff y^\textit{new} \\
	\le\int_0^1 y^\textit{new}g(y^\textit{new})\diff y^\textit{new}=\mu^{am}=\mu.
	\end{split}
	\end{equation}
	Here $g(y^\textit{new})$ is the pdf of $y^\textit{new}$ implied by 
	the restriction of $f^{am}(x)$ to the box $b$. The latter inequality holds since $y^\textit{new}\in [0,1]$ and $g(y^\textit{new})\ge0$.
\end{proof}
However, the condition $\mu=\mu^{am}$ of the proposition does not hold for all possible boxes $b$, unless $f(x)\equiv f^{am}(x)$. We experiment with this REDS modification in the case $L=N$ in Section~\ref{section:K-values}.

\subsection{Discussion of the Statistical Derivation}
To avoid restrictive assumptions on the true function $f(x)$, we made certain simplifications.
First, we assume that the box $b$ and the number of points it contains, $n'$, are both fixed, while the points in $D$ are sampled at random. 
In reality, only $n'$ is fixed at each iteration, and the box boundary varies to include exactly $n'$ points.
This variation is low once $p(x)$ is close to uniform.
Allowing the box boundary to vary with different realizations of $D$ would make MSE estimates (\ref{eq:bv_prim}) and (\ref{eq:bv_prim_ml}) incomparable. 
Second, one usually uses so-called space-filling designs~\cite{Santner2003} to form a dataset $D$ rather than ``brute force'' random sampling., e.g., Latin hypercube sampling~\cite{Kleijnen2015}. 
Generally, this would result in lower variance values than estimated with (\ref{eq:bv_prim}) or (\ref{eq:bv_prim_ml}). 
With these simplifications, our analysis explains the experimental results sufficiently well.

\section{Complexity Analysis} 
\label{section:complexity} 
We now derive the time complexities of the algorithms introduced above. 
We do so with respect to characteristics $N$ and $M$ of a dataset $D$, the hyperparameters $\alpha$, $bs$, $m$, $Q$ of the subgroup-discovery algorithms and the hyperparameter $L$ of REDS. 

\textbf{PRIM and PRIM with Bumping}.
PRIM requires quantile computation for each input. A straightforward approach is to sort the values of each input once, although other variants are possible~\cite{DBLP:journals/cacm/Hillmore62b}, in $O(M\cdot N\log N)$. 
After sorting, PRIM requires $O(M\cdot N(1-\alpha)^k)$, $k=0,1,\dots$ time for each iteration. 
In the worst case, PRIM stops when it runs out of data, and its overall complexity is $O(M\cdot N(\log N+1/\alpha))$. PRIM with bumping executes PRIM $Q$ times using a subset of $m$ attributes; its complexity is $O(N(M\log N+(Q\cdot m)/\alpha))$. 

\textbf{BI}.
The \textsc{BestIntervalWRAcc} subroutine requires sorted input values. 
Sorting is done once for each input, requiring $O(M\cdot N\log N)$ time. 
With sorted inputs, the complexity of \textsc{BestIntervalWRAcc} is $O(N)$~\cite{MampaeyNFK12}. 
The number of ``while'' cycles (see Algorithm~\ref{algorithm:BSBI}) is dataset-dependent; usually it does not differ much from the search depth $m$. 
The number of external ``for'' cycles is restricted by $bs$. 
All this leads to $O(M\cdot N(\log N+m\cdot bs))$ time for $\textsc{BI}$. 

\textbf{Metamodels}.
For the metamodels we consider, the time complexities are as follows. 
Random forest is a fixed number of decision trees. 
A binary decision tree algorithm, e.g., CART~\cite{DBLP:books/wa/BreimanFOS84}, (1) sorts all attributes once and (2) at each depth level $k$ iterates through $N$ values of each input to find the best split. 
A depth of a balanced decision tree is $\log_2 N$~\cite{hastie2009elements}. 
Hence, the training time of a random forest is in $O(\psi(M)\cdot N\log N)$. Here $\psi(M)$ is the hyperparameter of random forest~--- the number of inputs to consider when looking for the best split, $\psi(M)\le M$. 
Training XGBoost takes $O(M\cdot N\log N)$~\cite{DBLP:conf/kdd/ChenG16}.  
The complexity of non-linear SVM is between $O(M\cdot N^2)$ and $O(M\cdot N^3)$~\cite{bottou2007support,DBLP:journals/neco/Chapelle07}. 

\textbf{REDS}.
The complexity of REDS is the sum of the ones of a subgroup-discovery algorithm $SD$ and a metamodel $AM$. 
For instance, if $AM$ is XGBoost and $\textsc{SD}$ is PRIM, REDS discovers a scenario in $O(M(N\log N + L\log L+L/\alpha))$ time, where $L$ is the number of examples in $D^\textit{new}$, as before.

\section{Experimental Methodology}
\label{section:experimentalsetup} 
This section describes requirements for a meaningful evaluation of scenario-discovery algorithms, followed by our experimental setup. 
We start with the principles. 
After introducing naming conventions, we then describe our choices regarding hyperparameters of algorithms, data sets, and experiment design.

\subsection{Evaluation Principles}
\label{subsection:survey}
To come up with an  evaluation procedure that is conclusive, we have surveyed the literature on scenario and subgroup discovery. We have found that some existing evaluations do not comply with three principles of a conclusive evaluation, 
namely (1) using many datasets, (2) optimizing the hyperparameters of methods, and (3) using independent data for testing. 
Diversity of datasets is essential to ensure the generality of results. We show the importance of the latter two principles in the following example.

\begin{figure}[t]
	\centering
	\includegraphics[width=2.1in]{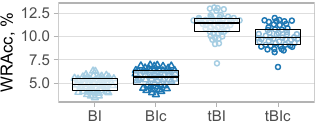}
	\caption{Demonstration. Evaluation of BI.}
	\label{fig:demo}
\end{figure}

\begin{example}
	We generated data sets $D_i$, $i=1,\dots,50$, $N=400$, and a test data set $D^\textit{test}$, $\lvert D^\textit{test}\rvert=20000$ using the same function ``morris''~\cite{saltelli2000sensitivity}.
	To each data set $D_i$, we have applied the BI algorithm to find a subgroup. 
	We tried both $m=M$ and choosing the value of $m$ via 5-fold cross-validation.
	The results are in Figure~\ref{fig:demo}, where the notation is as follows. 
	If the quality is measured on the train data $D_i$ and not on $D^\textit{test}$, we add the letter ``t''.
	Next, if $m$ is optimized via cross-validation, we write ``c''. 
	For instance, ``tBIc''means that we evaluated the output of the BI algorithm on the train data and used a procedure for hyperparameter optimization.
	For each item, ``BI'', ``BIc'', ``tBI''and ``tBIc'', there are 50 points, equal to the number of data sets $D_i$ we generated.
	
	The over-plotted boxes are the quartiles of WRAcc obtained in these 50 experiments. 
	First, one can see that hyperparameter optimization improves the results: WRAcc for ``BIc''is generally higher than for ``BI''. 
	Next, not using the independent data set $D^\textit{test}$ leads to (a) overly optimistic quality assessment~--- ``tBI'', ``tBIc''have higher WRAcc values than ``BI'', ``BIc'', and (b) misleading rankings of methods~--- WRAcc of ``tBI''is higher than ``tBIc'', but on the test data, the ranking is the other way around. 
\end{example}

Although these three principles represent good practice, various subgroup-discovery studies do not follow them. For instance, \cite{DBLP:journals/jair/GambergerL02,MampaeyNFK12,DBLP:journals/tfs/JesusGHM07,DBLP:journals/soco/CarmonaGJNJ11,Kwakkel2016,Kwakkel2016a,DBLP:journals/kbs/CarmonaJH18,Friedman1999,DBLP:journals/aai/KavsekL06,DBLP:journals/pvldb/GebalyAGKS14} use fewer than four data sets. 
The authors of~\cite{DBLP:journals/datamine/GrosskreutzR09,MampaeyNFK12,vollmer2019informative,DBLP:conf/pakdd/MillotCB20,DBLP:journals/jcst/Helal16,Dalal2013,Kwakkel2016a,DBLP:journals/eswa/RomeroGVJH09,DBLP:journals/kbs/CarmonaJH18,DBLP:journals/pvldb/GebalyAGKS14} do not evaluate methods on independent test data. 
Various papers just mentioned and~\cite{DBLP:journals/jmlr/LavracKFT04,DBLP:conf/eurogp/LunaRRV13,DBLP:journals/isci/RodriguezRRA12} with few exceptions (\cite{DBLP:journals/tfs/JesusGHM07} and~\cite{DBLP:journals/soco/CarmonaGJNJ11}) use the pre-fixed values of hyperparameters in their comparative experiments, usually proposed by the inventors of the respective approach.

\subsection{Naming}
To refer to methods evaluated, we adopt the following conventions. ``P''stands for the peeling step of PRIM,
''PB''means PRIM with bumping. ``BI''and ``BI5''stand for the BI algorithm with $bs=1$ and $bs=5$. 
If hyperparameters of a scenario-discovery algorithm are optimized, we add the letter ``c'', e.g., ``BIc''.
REDS starts with ``R''; ``f'', ``x'', and ``s''refer to different metamodels, namely random \textbf{f}orest, \textbf{X}GBoost, and \textbf{S}VM with RBF kernel~\cite{vert2004primer}. If we use the modification of REDS discussed in Section~\ref{subsection:reds} where  $y^\textit{new}_i=f^{am}(x^\textit{new}_i)$, we add ``p''(\textbf{p}robabilities). For instance, ``RPxp''denotes modified REDS with PRIM as an $\textsc{SD}$ algorithm and XGBoost as $AM$.
 
\subsection{Data Sets}
\label{subsection:data} 

\begin{table}[]
	\caption{Functions for experimental study}
	\label{tab:app_dgps}
	\centering
	\begin{tabular}{lccccc}
		\toprule
		\textbf{function}     & \textbf{M}  & \textbf{I} & \textbf{reference} & \textbf{thr} & \textbf{share (\%)} \\ \midrule
		1                & 5  & 2  & \cite{Dalal2013}    &   na  & 47.6  \\
		2                & 5  & 2  & \cite{Dalal2013}    &   na  & 25.7  \\
		3                & 5  & 2  & \cite{Dalal2013}    &   na  & 8.2  \\ 
		4                & 5  & 2  & \cite{Dalal2013}    &   na  & 18  \\ 
		5                & 5  & 2  & \cite{Dalal2013}    &   na  & 8  \\ 
		6                & 5  & 2  & \cite{Dalal2013}    &   na  & 8.1  \\ 
		7                & 5  & 2  & \cite{Dalal2013}    &   na  & 35  \\ 
		8                & 5  & 2  & \cite{Dalal2013}    &   na  & 10.9  \\ 
		102              & 15 & 9  & \cite{Dalal2013}    &   na  & 67.2  \\ 
		borehole         & 8  & 8  & \cite{simulationlib}    & 1000    & 30.9  \\ 
		dsgc             & 12 & 12 & \cite{Schafer2015}    &   na  & 53.7  \\
		ellipse          & 15 & 10 & our    &   0.8  & 22.5  \\ 
		hart3            & 3  & 3  & \cite{simulationlib}    & $-1$    & 33.5  \\ 
		hart4            & 4  & 4  & \cite{simulationlib}    & $-0.5$    & 30.1  \\ 
		hart6sc          & 6  & 6  & \cite{simulationlib}    & 1    & 22.6  \\ 
		ishigami         & 3  & 3  & \cite{simulationlib}    & 1    & 25.5  \\ 
		linketal06dec    & 10 & 8  & \cite{simulationlib}    & 0.15    & 25.3  \\ 
		linketal06simple & 10 & 4  & \cite{simulationlib}    & 0.33    & 28.5  \\ 
		linketal06sin    & 10 & 2  & \cite{simulationlib}    & 0    & 27.2  \\ 
		loepetal13       & 10 & 7  & \cite{simulationlib}    & 9    & 38.9  \\
		moon10hd         & 20 & 20 & \cite{simulationlib}    & 0    & 42.1  \\ 
		moon10hdc1       & 20 & 5  & \cite{simulationlib}    & 0    & 34.2  \\ 
		moon10low        & 3  & 3  & \cite{simulationlib}    & 1.5    & 45.6  \\
		morretal06       & 30 & 10 & \cite{simulationlib}    & $-330$    & 34.5  \\
		morris           & 20 & 20 & \cite{saltelli2000sensitivity}    & 20    & 30.1  \\
		oakoh04          & 15 & 15 & \cite{simulationlib}    & 10    & 24.9  \\ 
		otlcircuit       & 6  & 6  & \cite{simulationlib}    & 4.5    & 22.5  \\ 
		piston           & 7  & 7  & \cite{simulationlib}    & 0.4    & 36.8  \\ 
		soblev99         & 20 & 19 & \cite{simulationlib}    & 2000    & 41.3  \\ 
		sobol            & 8  & 8  & \cite{saltelli2000sensitivity}   & 0.7    & 39.2  \\ 
		welchetal92      & 20 & 18 & \cite{simulationlib}    & 0    & 35.6  \\ 
		willetal06       & 3  & 2  & \cite{simulationlib}    & $-1$    & 24.9  \\ 
		wingweight       & 10 & 10 & \cite{simulationlib}    & 250   & 37.8  \\ 
		TGL 			 & 9  & na  & \cite{Bryant2010}       & na    & 10.1  \\
		lake 			 & 5  & na  & \cite{DBLP:journals/envsoft/Kwakkel17} & na    & 33.5  \\
		\bottomrule
	\end{tabular}
\end{table}

We use 32 functions for our experiments,
most of which are commonly used in the metamodeling domain, one simulation model, the decentral smart grid control, and two datasets studied in scenario discovery research. 
We now describe these data sources. 

First, we have implemented functions 1--8 and 102 from~\cite{Dalal2013} following the descriptions in the paper. 
These are ``noisy''functions representing stochastic simulations.
The functions ``morris''and ``sobol''are implemented in the R package ``sensitivity''\footnote{\url{https://cran.r-project.org/web/packages/sensitivity/}}.
Next, we use the R implementations of the other functions from~\cite{simulationlib}.
We keep the original names of the functions as provided together with the implementation. 
We also introduce the function ``ellipse''$f_e(x)=\sum_{j=1}^{15}w_j\cdot(x_j-c_j)^2$
where $w_j\in[0,1]$, $c_j\in[0,1]$ are some constants, $w_j=0$ if $j>10$.
We binarized the output of real-valued functions by specifying the threshold $\textit{thr}$, so that $y=1$ if the output is below it and $y=0$ otherwise. 
This is common in scenario discovery \cite{Bryant2010}.

Our simulation model, ``dsgc''~\cite{Schafer2015}, is a novel approach facilitating demand response in electrical grids. 
We configured the model to have 12 inputs and one output which indicates the grid stability. For brevity we also refer to ``dsgc''as ``function''. 

To show that REDS can improve the results obtained from third-party data, we add datasets ``TGL''and ``lake''from publications on scenario discovery \cite{Bryant2010,DBLP:journals/envsoft/Kwakkel17}. 
''TGL''consists of 882 examples; from the ``lake''dataset we use the first 1000 examples.

Table~\ref{tab:app_dgps} lists all data sources used in our study. Here, $M$ is the number of inputs. $I\le M$ is the number of inputs affecting the output. 
The threshold values are in column ``thr''. The functions which already output $y\in\{0,1\}$ have the values ``na'' in this column. 
The expected share of outcomes $y=1$ with uniform distribution of points is in column ``share''. 

\subsection{Hyperparameters}
\label{subsection:hyerparameters}

We experiment with algorithms with both ``default''and optimized hyperparameters. 
In the following, we explain our design choices. Table~\ref{tab:hyperparameters} is a summary. The symbol ``$\ast$''in this table denotes any combination of characters excluding ``c'', e.g., ``xp''.

\subsubsection{PRIM}
There is no explicit agreement on the default value of the peeling parameter $\alpha$. In~\cite{Friedman1999,Kwakkel2016}, the range $[0.05,0.1]$ of values is recommended. 
Friedman and Fisher~\cite{Friedman1999} use $\alpha=0.1$ in their experiments, whereas Kwakkel and Jaxa{-}Rozen~\cite{Kwakkel2016a} experiment with $\alpha\in\{0.01,0.025,0.05,0.1\}$ but do not recommend a value in this set. 
Existing implementations of PRIM use either $\alpha=0.1$\footnote{\url{https://cran.r-project.org/web/packages/sdtoolkit}} or $\alpha=0.05$\footnote{\label{software:prim}\url{https://github.com/quaquel/EMAworkbench},\newline\url{https://cran.r-project.org/web/packages/prim/index.html}} as default. 
We set $\alpha=\alpha_d=0.05$ for all experiments except for those with ``TGL''dataset, where we use $\alpha=0.1$ in line with the previous research~\cite{Kwakkel2016a}. 
We also experiment with optimizing $\alpha$ for PRIM and PRIM with bumping. 
We do so by selecting the value from $\{0.03, 0.05, 0.07, 0.1, 0.13, 0.16, 0.2\}$ that performs best in the 5-fold cross-validation. 
Next, we set $mp=20$ so that $mp\cdot\alpha_d=1$. 

The hyperparameter values for PRIM with bumping were not stated in the respective paper \cite{Kwakkel2016}. 
$Q$ regulates the number of bootstrap repetitions. 
Intuitively, larger values of $Q$ are better but slow down the algorithm proportionally. 
We set $Q=50$. 
Hyperparameter $m$ restricts the number of inputs defining the box. 
We set its default value to the number of inputs in each dataset $m=M$, where $M$ is the number of inputs as before. 
To optimize hyperperameters of ``PBc'', we first select $\alpha$ from $\{0.03, 0.05, 0.07, 0.1, 0.13, 0.16, 0.2\}$ with conventional PRIM
and then select $m$ from the set $\{M-k\lceil M/6\rceil\}$, $k\in\{0\}\cup\mathbb{N}$, $k\lceil M/6\rceil<M$ via cross-validation.

\subsubsection{BI} Hyperparameters of the BI algorithm are the beam size $bs$ and $m$, the number of inputs defining a subgroup. 
A higher $bs$ value results in the evaluation of more candidate subgroups but increases the run time proportionally. 
We experiment with $bs\in\{1,5\}$. 
As with PRIM with bumping, we set the default value $m=M$ or select $m$ from the set $\{M-k\lceil M/6\rceil\}$.

\subsubsection{REDS}
For REDS, one needs to specify the number of newly generated points $L$, a metamodel $AM$, and a scenario-discovery method $\textsc{SD}$. 
For REDS with PRIM, as for just PRIM, we either set $\alpha=0.05$ or use the same value as for ``Pc''.
Similarly, for REDS with BI as $\textsc{SD}$, we optimize the value $m$ in the same way as for ``BIc'', using the dataset $D$, not $D^\textit{new}$.
As default, we set $L=10^5$ when $\textsc{SD}$ is PRIM and $L=10^4$ when $\textsc{SD}$ is BI, much larger than any $N$ used. We also try different $L$ values in a particular experiment.
For metamodels $AM$ we use the default hyperparameter-optimization procedure of the package ``caret''\footnote{\url{http://topepo.github.io/caret/index.html}}. 

\begin{table}[!t]
	\caption{Hyperparameters' values}
	\label{tab:hyperparameters}
	\centering
	\begin{tabular}{lccccccccc}
		\toprule
		& P & Pc & PB & PBc & RP$\ast$ & BI & BI5 & BIc & RBIc$\ast$ \\ \midrule
		$\alpha$ & 0.05 & cv & 0.05 & cv & 0.05 & {\color[HTML]{9B9B9B} $\times$} & {\color[HTML]{9B9B9B} $\times$} & {\color[HTML]{9B9B9B} $\times$} & {\color[HTML]{9B9B9B} $\times$} \\
		$mp$ & 20 & 20 & 20 & 20 & 20 & {\color[HTML]{9B9B9B} $\times$} & {\color[HTML]{9B9B9B} $\times$} & {\color[HTML]{9B9B9B} $\times$} & {\color[HTML]{9B9B9B} $\times$} \\
		$Q$ & {\color[HTML]{9B9B9B} $\times$} & {\color[HTML]{9B9B9B} $\times$} & 50 & 50 & {\color[HTML]{9B9B9B} $\times$} & {\color[HTML]{9B9B9B} $\times$} & {\color[HTML]{9B9B9B} $\times$} & {\color[HTML]{9B9B9B} $\times$} & {\color[HTML]{9B9B9B} $\times$} \\
		$m$ & {\color[HTML]{9B9B9B} $\times$} & {\color[HTML]{9B9B9B} $\times$} & $M$ & cv & {\color[HTML]{9B9B9B} $\times$} & $M$ & $M$ & cv & cv \\
		$bs$ & {\color[HTML]{9B9B9B} $\times$} & {\color[HTML]{9B9B9B} $\times$} & {\color[HTML]{9B9B9B} $\times$} & {\color[HTML]{9B9B9B} $\times$} & {\color[HTML]{9B9B9B} $\times$} & 1 & 5 & 1 & 1 \\
		$L$ & {\color[HTML]{9B9B9B} $\times$} & {\color[HTML]{9B9B9B} $\times$} & {\color[HTML]{9B9B9B} $\times$} & {\color[HTML]{9B9B9B} $\times$} & $10^5$ & {\color[HTML]{9B9B9B} $\times$} & {\color[HTML]{9B9B9B} $\times$} & {\color[HTML]{9B9B9B} $\times$} & $10^4$\\
		\bottomrule
	\end{tabular}
\end{table}

\subsection{Design of Experiments}
\label{subsection:doe}
In each experiment, we execute a subgroup discovery algorithm so that it returns a single box with BI or one sequence of nested boxes with PRIM. 
This is sufficient since (1) the quality of a set of scenarios is an aggregate of their individual qualities and (2) subsequent steps of the covering approach allowing to discover several scenarios (see Section~\ref{section:prim}) are conceptually equivalent.
This also is in line with existing research on scenario discovery, described in Section~\ref{subsection:scenario_discovery}.

For all functions, we experiment with data sets of sizes $N=\{200,\allowbreak400,\allowbreak800\}$. 
For ``morris'', $N=\{200,\allowbreak400,\allowbreak800,\allowbreak1600,\allowbreak3200\}$. 
To form the data sets $D$, we use the Halton sequence \cite{halton1964algorithm} sampling algorithm for ``dsgc''and Latin hypercube sampling from $[0,1]^M$ for all other functions. 
We use $D^{val}=D$.
For each function, we generate the test data $D^\textit{test}$ containing $20000$ points.
We run the experiment 50 times for each function and each value of~$N$. 
For ``TGL''and ``lake'', we run a 5-fold cross-validation independently 10 times.
To aggregate these results, we average the values of PR AUC and precision (for PRIM), WRAcc (for BI), and the numbers of restricted and irrelevantly restricted inputs. 
To estimate consistency, we compute $V_o/V_u$ (see Section~\ref{section:QM}) for each pair of (last) boxes from different runs and average the results. 
This is similar to the approach in~\cite{Domingos1997} to compute method stability, a measure akin to consistency.
We used a machine with 32 2GHz cores and 128GB of memory. 
We implemented methods in R and run the experiments in parallel, each one using a single core.

\section{Results}
\label{section:results}

\begin{figure}[t]
	\centering
	\subfloat
	{\includegraphics[width=2.7in]{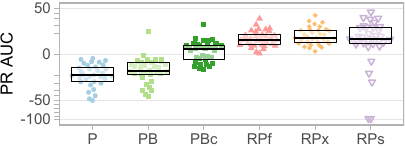}}
	\\
	\subfloat
	{\includegraphics[width=2.7in]{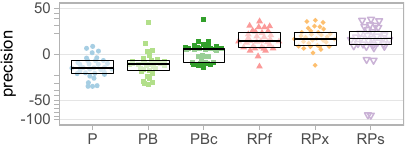}}
	\\
	\subfloat
	{\includegraphics[width=2.7in]{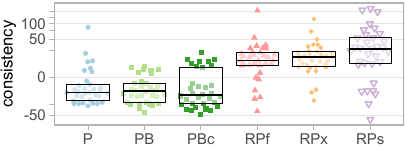}}
	\\
	\subfloat
	{\includegraphics[width=2.7in]{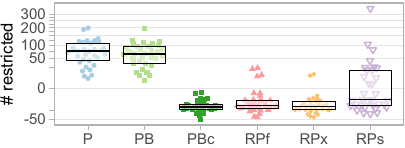}}
	\caption{Quality change in \% relative to ``Pc'', $N = 400$}
	\label{fig:prim_qual}
\end{figure}

We first compare the performance of methods across all functions. Then we further experiment with ``morris'' to study the influence of REDS hyperparameter $L$ and of data characteristics. Next, we experiment with ``TGL'' and ``lake'' to show how REDS improves scenarios discovered from third-party data when there is no simulation model available. 
Finally, we show that REDS also is an efficient semi-supervised subgroup-discovery method.

\subsection{Performance across All Functions}
\label{subsection:DGPs_results}

\begin{table}[t]
	\centering
	\caption{Quality of PRIM-based methods. All functions}
	\label{tab:prim_qual}
	\vspace*{-\baselineskip}
	\subfloat[Average PR AUC\label{tab:AUC}]{
		\begin{tabular}{lcccc|ccc} \toprule
			$N$ & \textbf{P} & \textbf{Pc} & \textbf{PB} & \textbf{PBc} & \textbf{RPf} & \textbf{RPx} & \textbf{RPs}\\ \midrule
			200 & 33.1 & 40 & 36.5 & 41.7 & \textit{45.1} & \textbf{45.5} & 40.7\\
			400 & 41.3 & 45.9 & 42.9 & 46.6 & \textit{48.6} & \textbf{49.4} & 44.8\\
			800 & 46.3 & 48.7 & 46.9 & 49 & \textit{50} & \textbf{50.8} & 45.9\\
			$\textrm{mor}_{800}$ & 14.5 & 23.5 & 15.3 & 23.1 & 27 & \textit{27.8} & \textbf{28.2}\\
			\bottomrule
		\end{tabular}
	}\\
	\subfloat[Average precision\label{tab:density}]{
		\begin{tabular}{lcccc|ccc} \toprule
			200 & 71.8 & 78.1 & 75 & 79.7 & \textit{85.6} & \textbf{86.4} & 82.2\\
			400 & 81.1 & 85.9 & 82.4 & 86.6 & \textit{91.3} & \textbf{92.6} & 88.3\\
			800 & 87.1 & 90.5 & 87.6 & 90.8 & \textit{93.8} & \textbf{95.3} & 90.2\\
			$\textrm{mor}_{800}$ & 58.8 & 73.8 & 57.7 & 70.5 & 86.3 & \textit{88.6} & \textbf{89.4}\\
			\bottomrule
		\end{tabular}
	}\\
	\subfloat[Average consistency\label{tab:consistency}]{
		\begin{tabular}{lcccc|ccc} \toprule
			200 & 40.5 & 45.2 & 40.2 & 42.6 & 50 & \textit{51.3} & \textbf{52.7}\\
			400 & 42.7 & 47.3 & 43.2 & 44.9 & 53.4 & \textit{53.6} & \textbf{56.8}\\
			800 & 45 & 49 & 45.1 & 46.4 & 55 & \textit{55.8} & \textbf{59.8}\\
			$\textrm{mor}_{800}$ & 14.8 & 17.2 & 11.5 & 10.7 & \textit{30.2} & 27.8 & \textbf{37.4}\\
			 \bottomrule			
		\end{tabular}
	}\\
	\subfloat[Average number of restricted inputs\label{tab:interpretability}]{
		\begin{tabular}{lcccc|ccc} \toprule
			200 & 7.79 & 4.27 & 7.3 & 3.34 & \textit{3.33} & \textbf{3.03} & 4.14\\
			400 & 7.75 & 4.32 & 7.46 & \textbf{3.54} & 3.72 & \textit{3.57} & 4.42\\
			800 & 7.48 & 4.38 & 7.35 & \textbf{3.75} & \textit{3.91} & 3.99 & 4.65\\
			$\textrm{mor}_{800}$ & 17.6 & 7.56 & 17.4 & \textbf{6.22} & \textit{6.88} & 7.7 & 7.44\\
			\bottomrule		
		\end{tabular}
	}\\
	\subfloat[Average number of irrelevantly restricted inputs\label{tab:interpretability_irr}]{
		\begin{tabular}{lcccc|ccc} \toprule
			200 & 2.83 & 0.38 & 2.47 & 0.12 & \textit{0.08} & \textbf{0.02} & 0.57\\
			400 & 2.77 & 0.33 & 2.51 & \textbf{0.07} & 0.11 & \textit{0.1} & 0.59\\
			800{\color{white}-----} & 2.43 & 0.37 & 2.28 & \textbf{0.09} & 0.13 & \textit{0.1} & 0.53\\
			\bottomrule
		\end{tabular}
	}	
\end{table}

We experiment with all 33 functions using continuous inputs and with 32 functions excluding ``dsgc'' using mixed (e.g., continuous and discrete) inputs.

\subsubsection{Continuous Inputs}
\label{section:continuous_inputs}
Table~\ref{tab:prim_qual} and Figure~\ref{fig:prim_qual} contain the results for PRIM-based methods across all 33 functions. 
The REDS variants ``RPxp'' and ``RPfp'' behaved similarly to ``RPx'' and ``RPf''.
The figures contain relative changes (in \%) of the measure values with respect to PRIM with optimized hyperparameters ``Pc''.
The vertical axes are scaled with the square root for better resolution.
Each point of the plot corresponds to an average value in $50$ experiments with one of 33 functions, for $N=400$. 
As before, the boxes show quartiles.
Higher values of PR AUC, precision, consistency, and fewer restricted inputs (higher interpretability) are better.
The first three rows in the tables are averaged metrics values for the 33 functions, for $N=\{200,400,800\}$. 
The rows ``$\textrm{mor}_{800}$'' are the results for the 20-dimensional ``morris'' function for $N=800$. 

For the PR AUC, precision, and consistency metrics, REDS beats the competitors. 
Its variant ``RPx'' leads to particularly good results.
According to pairwise post-hoc Friedman tests~\cite{DBLP:conf/gecco/OrzechowskiCM18}, ``RPx'' statistically outperforms conventional methods in terms of PR AUC and precision, with p-values $10^{-3}$ or less for $N=400$.
 
Regarding the number of restricted inputs and irrelevantly restricted inputs, the results with ``RPx'' and ``PBc'' are close.
On average, PR AUC and precision achieved for $N=800$ with ``Pc'' are lower than those obtained for $N=400$ with ``RPx''. 
Regarding consistency, results with REDS for $N=200$ already are better than any results of competitors even for $N=800$.
This means that our approach can reduce the number of simulation runs by 50--75\% for PR AUC, precision, and consistency. 
The Spearman correlation between the number of inputs $M$ and the relative PR AUC improvements of ``RPx'' over ``Pc'' for $N=400$ is 0.74. It shows that REDS is particularly powerful on high-dimensional functions which often require much time per simulation run.

\begin{table}[t]
	\centering
	\caption{Quality of BI-based methods. All functions}
	\label{tab:bi_qual}
	\vspace*{-\baselineskip}
	\subfloat[Average WRAcc\label{tab:wracc}]{
		\begin{tabular}{lccc|cc} \toprule
			\textbf{N} & \textbf{BI} & \textbf{BIc} & \textbf{BI5} & \textbf{RBIcfp} & \textbf{RBIcxp}\\ \midrule
			200 & 10.4 & 10.7 & 10.4 & \textit{11.2} & \textbf{11.3}\\
			400 & 10.9 & 11.2 & 10.9 & \textit{11.5} & \textbf{11.6}\\
			800 & 11.3 & 11.5 & 11.3 & \textit{11.6} & \textbf{11.8}\\
			$\textrm{mor}_{800}$ & 5.6 & 6.4 & 5.6 & \textit{6.7} & \textbf{7.2}\\
			\bottomrule
		\end{tabular}
	}\\
	\subfloat[Average consistency\label{tab:bi_consistency}]{
		\begin{tabular}{lccc|cc} \toprule
			200 & 58.5 & 60.7 & 58.5 & {\color{white}---}\textit{67.2}{\color{white}---} & {\color{white}---}\textbf{68.4}{\color{white}---}\\
			400 & 64.5 & 66.8 & 64.9 & \textit{72.2} & \textbf{73.3}\\
			800 & 70.3 & 71.9 & 70.4 & \textit{75.7} & \textbf{77.5}\\
			$\textrm{mor}_{800}$ & 43.1 & 47.9 & 43.4 & \textit{51} & \textbf{63}\\
			\bottomrule
		\end{tabular}
	}\\
	\subfloat[Average number of restricted inputs\label{tab:bi_restricted}]{
	\begin{tabular}{lccc|cc} \toprule
		200 & 5.07 & \textit{3.17} & 5.27 & {\color{white}---}\textbf{3.15}{\color{white}---} & {\color{white}---}3.35{\color{white}---}\\
		400 & 6 & \textit{3.28} & 6.04 & \textbf{3.2} & 3.32\\
		800 & 6.57 & \textit{3.26} & 6.61 & \textbf{3.19} & \textit{3.26}\\
		$\textrm{mor}_{800}$ & 15.1 & 5.08 & 15.1 & \textbf{4.78} & \textit{5}\\ 
		\bottomrule
	\end{tabular}
	}\\
	\subfloat[Average number of irrelevantly restricted inputs\label{tab:bi_irrel}]{
	\begin{tabular}{lccc|cc} \toprule
		200 & 0.79 & \textbf{0.05} & 0.9 & {\color{white}---}0.09{\color{white}---} & {\color{white}---}\textit{0.08}{\color{white}---}\\
		400 & 1.4 & \textit{0.04} & 1.4 & \textbf{0.03} & \textbf{0.03}\\
		800{\color{white}-----} & 1.74 & \textit{0.05} & 1.79 & \textbf{0.04} & \textit{0.05}\\ 
		\bottomrule
	\end{tabular}
	}
\end{table}

\begin{figure}[t]
	\centering
	\subfloat
	{\includegraphics[width=0.95in]{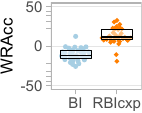}}
	\quad
	\subfloat
	{\includegraphics[width=0.95in]{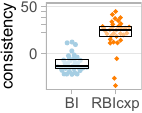}}
	\quad
	\subfloat
	{\includegraphics[width=1.05in]{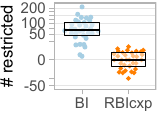}}
	\caption{Quality change in \% relative to ``BIc'', $N = 400$}
	\label{fig:bi_qual}
\end{figure}

Table~\ref{tab:bi_qual} and Figure~\ref{fig:bi_qual} feature the results for BI-based methods. 
As before, vertical axes on the plots are square root-scaled.
Points represent relative quality changes with respect to ``BIc''. 
Similarly to the PRIM-based methods, hyperparameter optimization improves the results. WRAcc, consistency, and interpretability of ``BIc'' (zero ordinate) are better than their counterparts obtained with ``BI''. 
REDS statistically outperforms the baselines. The p-value of the pairwise post-hoc Friedman test between ``RBIcxp'' and ``BIc'' for $N=400$ is $10^{-3}$. 
The quality gain of REDS increases with the number of inputs. The Spearman correlation between $M$ and the relative WRAcc improvements of ``RBIcxp'' over ``BIc'' for $N=400$ is 0.77.
On average, REDS needs more than two times fewer simulations to achieve WRAcc or consistency similar to ``BIc'', while its interpretability is comparable to the one of ``BIc''.

\begin{figure}[t]
	\centering
	\subfloat
	{\includegraphics[width=1.75in]{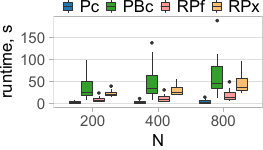}}
	\quad
	\subfloat
	{\includegraphics[width=1.45in]{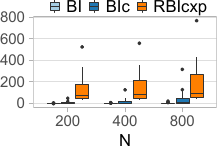}}
	\caption{Runtimes}
	\label{fig:runtimes}
\end{figure}

Figure~\ref{fig:runtimes} shows runtimes contingent on $N$. As before, the plots are based on the results on 33 functions, each averaged across 50 experiments. In our case, many methods scale sublinearly with dataset size. For REDS, this likely means that, for small $N$, the $L$-dependent terms (see Section~\ref{section:complexity}) dominate the complexity. For baselines, the sub-linear behavior implies that the cost of sorting is negligible for small $N$. 
In all cases, the runtime is less than 800 seconds, often much less. Given that REDS requires 2--4 times fewer simulations than competitors on average, this means that REDS is already faster when a single simulation lasts longer than two seconds at $N= 400$. In many settings however, a simulation takes hours to days~\cite{yilmaz2019reducing,Wang2007}.

\subsubsection{Mixed Inputs}
\label{subsubsection:mixed_inputs}

\begin{figure}[t]
	\centering
	\subfloat
	{\includegraphics[width=1in]{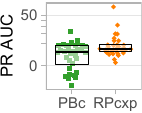}}
	\quad
	\subfloat
	{\includegraphics[width=1in]{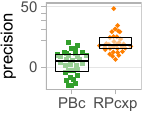}}
	\quad
	\subfloat
	{\includegraphics[width=1in]{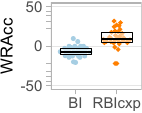}}
	\caption{Quality change in \% relative to ``Pc''/''BIc'', $N = 400$}
	\label{fig:discr}
\end{figure}

So far, all inputs of our functions were continuous; this is a common case in scenario discovery. We now show that REDS finds better scenarios in the case of mixed inputs. To this end, we sample the values of \emph{even} inputs i.i.d.\ from the set $\{0.1,\allowbreak0.3,\allowbreak0.5,\allowbreak0.7,\allowbreak0.9\}$. 
The REDS modification ``RPcxp'' performed best among PRIM-based methods. ``RBIcxp'' was better than its BI-based competitors. Figure~\ref{fig:discr} shows a relative quality gain of these methods with respect to ``Pc'' or ``BIc'' for $N=400$. The p-values of the pairwise post-hoc Friedman test between REDS and competitors for the metrics presented are $0.017$ or less, i.e., the results of REDS are significantly better.

\subsection{Experiments with ``morris''}
\label{subsection:DSGC_results}
We now carry out further experiments with ``morris'' to study the influence of randomness in the data set $D$, of its size $N$, and of the values of the hyperparameter $L$ of REDS.

\subsubsection{Peeling Trajectories and Variance of the Results}

\begin{figure}[t]
	\centering
	\subfloat
	{\includegraphics[width=1.5in]{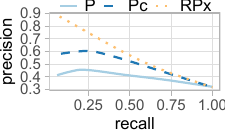}}
	\qquad
	\subfloat
	{\includegraphics[width=1.5in]{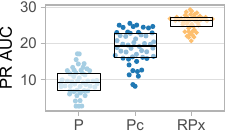}}
	\caption{Peeling trajectories \& PR AUC, ``morris'', $N = 400$}
	\label{fig:p_b}
\end{figure}

Figure~\ref{fig:p_b} plots the peeling trajectories for different methods and $N = 400$ smoothed across 50 repetitions and PR AUC values in these 50 experiments. 
On the first plot, the curve produced with ``RPx'' dominates the ones obtained with competitors, ``P'' or ``Pc''. 
That is, both precision and recall are higher.
``RPx'' yields a \emph{significant} improvement in PR AUC over ``Pc'' as the right plot suggests; the p-value $<10^{-15}$ of the Wilcoxon-Mann-Whitney test~\cite{hollander1973nonparametric} confirms this. Comparing ``RBIcxp'' to ``BIc'' results in similar findings.

\subsubsection{Dependence on $N$ and $L$}
\label{section:K-values}

\begin{figure}[t]
	\centering
	\subfloat
	{\includegraphics[width=3.2in]{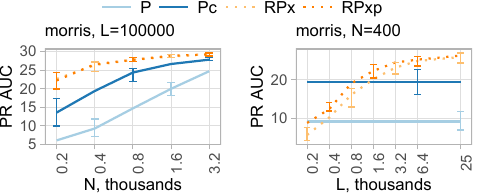}}
	\\
	\subfloat
	{\includegraphics[width=3.2in]{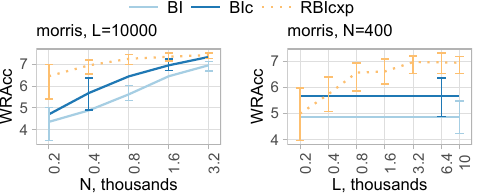}}
	\caption{Influence of $N$/$L$, function ``morris'', $N=400$}
	\label{fig:lc_ngen}
\end{figure}

The horizontal axes in Figure~\ref{fig:lc_ngen} are in logarithmic scale.
The lines are median values of 50 experiments for each value of $N$ or $L$; 
error bars show the interquartile ranges. 
The plots on the left side show the influence of the number of simulations $N$. The quality of scenarios grows with $N$; for large $N$ it approaches a data dependent limit, forming so-called learning curves. 
The learning curves of REDS dominate the ones of competing methods, as expected, cf.\  Figure~\ref{fig:generalization_speed}. 
The plots on the right side show the influence of the number of newly generated examples $L$. 
Observe that REDS uses only $L$ newly generated and labeled points, ignoring the initial data. 
``RPxp'' outperforms ``P'' with the same $\alpha=0.05$ by much even when $L=N=400$. 
This means that labeling points with $f^{am}(x)$ lets PRIM find a better scenario than labels $\{0,1\}$ from the original simulation model. 
This result confirms our statistical analysis.

\subsection{Scenario Discovery from Third-Party Data}
\label{subsection:legacy}

\begin{figure}[t]
	\centering
	\includegraphics[width=3in]{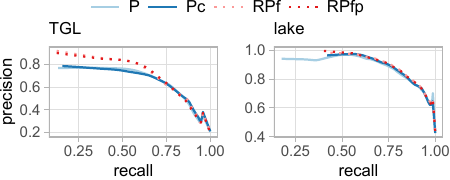}
	\caption{Peeling trajectories for ``TGL'' and ``lake'' datasets}
	\label{fig:peeling_real}
\end{figure}

\begin{table}[t]
	\centering
	\caption{Performance on ``TGL'' and ``lake'' datasets}
	\label{tab:reds_on_real}
	\begin{tabular}{lccc|ccc} \toprule
		 & \multicolumn{3}{c}{TGL} & \multicolumn{3}{c}{lake} \\ 
		& \textbf{Pc} & \textbf{RPf} & \textbf{RPfp} & \textbf{Pc} & \textbf{RPf} & \textbf{RPfp} \\ \midrule
		PR AUC & 61.2 & \textit{65.8} & \textbf{65.9} & 58.1 & \textit{58.9} & \textbf{59.4} \\
		precision & 84.4 & \textit{94.4} & \textbf{94.8} & 97.4 & \textit{98.2} & \textbf{99.1} \\
		consistency & 41.3 & \textit{60.6} & \textbf{67.2} & 74.9 & \textit{86.7} & \textbf{94.5} \\
		\# restricted & 4.94 & \textbf{3.76} & \textit{3.84} & \textit{2.94} & \textbf{2.92} & 3 \\
		\bottomrule
	\end{tabular}
\end{table}

We demonstrate how REDS improves scenario discovery from third-party datasets using ``TGL'' and ``lake''.
Figure~\ref{fig:peeling_real} shows the peeling trajectories smoothed across 50 experiments. 
REDS improves the results of PRIM in high-precision areas. 
Table~\ref{tab:reds_on_real} lists averaged metrics value. 
For both ``TGL'' and ``lake'', REDS finds much more stable scenarios than ``Pc'', as high consistency values suggest. 
For ``TGL'', REDS also yields much better values for the other metrics. 
``RPx'' and ``RPxp'' (not presented here) behave somewhat worse than ``RPf'' and ``RFfp'', but better than ``Pc''. 

\subsection{REDS as a Semi-supervised Method}
\label{section:reds_ssl} 

\begin{figure}[t]
	\centering
	\subfloat
	{\includegraphics[width=1in]{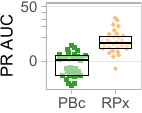}}
	\quad
	\subfloat
	{\includegraphics[width=1in]{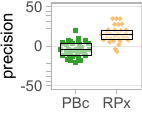}}
	\quad
	\subfloat
	{\includegraphics[width=1in]{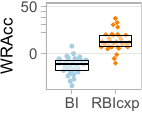}}
	\caption{Quality change in \% relative to ``Pc''/''BIc'', $N = 400$} 
	\label{fig:ssl}
\end{figure}

Learning methods using both labeled and unlabeled data for training are semi-supervised~\cite{3603,zhu2005semi}. 
As explained in Section~\ref{section:proposed}, REDS is suitable for the semi-supervised scenario discovery when both kinds of data follow the same distribution $p(x)$. 
Conceptually, the only difference of this setting to scenario discovery is that $p(x)$ is not confined to be uniform. 
To test REDS as a semi-supervised method, we sample all inputs of the functions independently from a logit-normal distribution with $\mu=0$ and $\sigma=1$. 
We used the same values for $thr$ as in Table~\ref{tab:app_dgps} and kept 30 functions for which the share of ``interesting'' ($y=1$) examples remained greater than 5\%. 
The results in Figure~\ref{fig:ssl} are similar to the ones from Section~\ref{section:continuous_inputs}~--- REDS is better than the competitors in a semi-supervised setting.

\section{Future Research}
\label{section:future}
Despite the example supporting PRIM from Section~\ref{section:discussion}, there are up to now no experimental comparison of subgroup-discovery algorithms regarding their suitability for scenario discovery.
In future work, we plan to perform a user study with domain experts to assess the practical usability of various subgroup-discovery algorithms and benefits of REDS for scenario discovery.

Next, we have proposed REDS to decrease the number of simulations needed for learning scenarios, i.e., to minimize the labeling effort. 
There exist techniques with a similar target, known under the names
active learning~\cite{Settles2010}, adaptive sampling~\cite{Gorissen2010} or selective sampling~\cite{badarna2019selective}.
Active learning (AL) allows an ML algorithm to iteratively select the most ``informative'' instances to be labeled next. 
For box-producing models, several AL methods have been proposed for decision trees~\cite{DBLP:journals/tkde/DimitriadouPD16,DBLP:conf/icml/LewisC94,DBLP:conf/ecml/DwyerH07}.
For subgroup discovery, we are not aware of any such method. Hence, we see the development of such an AL technique as a direction for future research. 

On the other hand, REDS can already benefit from existing active learning research. 
Namely, several AL methods exist for many machine learning models such as random forest~\cite{badarna2019selective}, support vector machine~\cite{tong2001support}, XGBoost~\cite{DBLP:journals/access/XueW20b} and others~\cite{Settles2010}, which one may use as an intermediate metamodel $AM$ in REDS. 
Combining REDS with active learning is another future research direction.

The success of REDS might depend on the complexity of the boundary in the input space that separates examples with $y=1$ and $y=0$. 
So far we have used dimensionality as a proxy for complexity. 
In the future, we aim to propose better complexity measures and study the influence of complexity on REDS performance.

\section{Conclusions}
\label{section:conclusion}
Simulations allow studying the behavior of complex systems. 
Scenario discovery, the topic of this paper, is the process of using simulations to gain insights regarding this behavior. 
Subgroup-discovery methods, PRIM in particular, are in use for scenario discovery. 
These methods isolate conditions for system behavior of interest, referred to as ``scenarios''.
As we have shown, the disadvantage of these algorithms
is their need for relatively many simulation runs. 
In this paper, we have proposed an improvement of scenario discovery. 
Based on data simulations have generated, our method, REDS, first trains a statistical model. 
This model then replaces the simulation model to label much more data for scenario discovery. 
We have shown that one can expect our approach to be superior with statistical arguments and with exhaustive experiments. 
Specifically, REDS requires 50--75\% fewer points on average than conventional methods, to
yield scenarios of comparable quality.
In addition, we demonstrated that REDS can improve scenario discovery from third-party data and can be used for semi-supervised subgroup discovery.

\begin{acks}
	This work was supported by the \grantsponsor{}{German Research Foundation (Deutsche Forschungsgemeinschaft)}{}, Research Training Group \grantnum{}{2153}: ``Energy Status Data~--- Informatics Methods for its Collection, Analysis and Exploitation''. 
	We thank the anonymous referees for their valuable comments and helpful suggestions. We also thank Georg Steinbu{\ss}, Edouard Fouch\'e, Pavel Obraztcov and Tien Bach Nguyen for fruitful discussions and other support.
\end{acks}

\balance
\bibliographystyle{ACM-Reference-Format}
\bibliography{References}

\end{document}